\documentclass{article}

\usepackage{arxiv}

\usepackage[utf8]{inputenc}
\usepackage[T1]{fontenc}
\usepackage{hyperref}
\usepackage{url}
\usepackage{booktabs}
\usepackage{amsfonts}
\usepackage{nicefrac}
\usepackage{microtype}
\usepackage{graphicx}
\usepackage{amsmath,amssymb}
\usepackage{mathtools}
\usepackage{bm}
\usepackage{mathrsfs}
\usepackage{amsthm}
\usepackage{multirow}
\usepackage{float}
\usepackage{xcolor}
\usepackage{array}
\usepackage{enumitem}
\usepackage{caption}

\theoremstyle{plain}
\newtheorem{theorem}{Theorem}[section]

\newtheorem{proposition}[theorem]{Proposition}

\theoremstyle{definition}

\newcommand{\R}{\mathbb{R}}

\newcommand{\Eucl}{\mathbb{E}^2}
\newcommand{\Sph}{\mathbb{S}^2}
\newcommand{\Hyp}{\mathbb{H}^2}
\newcommand{\mM}{\mathcal{M}}
\newcommand{\mL}{\mathcal{L}}
\newcommand{\bP}{\mathbf{P}}
\newcommand{\bQ}{\mathbf{Q}}
\newcommand{\bG}{\mathbf{G}}
\newcommand{\bx}{\mathbf{x}}
\newcommand{\bz}{\mathbf{z}}
\newcommand{\balpha}{\bm{\alpha}}

\title{Adaptive Interpolatory Curve Subdivision with Learned Local Angles}

\author{
  Hassan Ugail \\
  Centre for Visual Computing and Intelligent Systems \\
  University of Bradford \\
  Bradford, United Kingdom \\
  \\
  \And
  Newton Howard \\
  School of Individualized Study \\
  Rochester Institute of Technology \\
  New York, United States \\
}

\begin{document}
\maketitle

%-----------------------------------------------------------------------
\begin{abstract}
Curve subdivision is pivotal in computer graphics for generating smooth geometric objects from control polygons. Interpolatory subdivision is especially attractive because the refined curve is guaranteed to pass through the designer's control points. Classical four-point and six-point schemes preserve this property, but their behaviour is governed by a single global tension parameter, limiting their ability to adapt across flat regions, sharp turns and varying local geometries. We introduce an adaptive local-angle formulation that keeps the interpolatory structure intact while learning how each new vertex should be inserted. A compact edge-wise predictor assigns one insertion angle per edge, while the original vertices are copied exactly at every refinement level. Interpolation is therefore a structural property of the operator and does not depend on the trained weights. The same predictor is used with geometry-specific geodesic primitives on the Euclidean plane, the two-sphere and the Poincar\'e disk. Under a matched-density evaluation protocol, the method reduces nearest-neighbour error by factors of five to seventeen over the best validation-tuned fixed-tension baseline, and by about $1.8$ over centripetal Catmull--Rom in the Euclidean case. It also substantially reduces bending energy and tangent roughness, while remaining competitive with separately trained per-geometry models.
\end{abstract}


\section{Introduction}
\label{sec:intro}

Subdivision schemes have been central to curve and surface design since the 1970s, generating smooth geometric objects from piecewise-linear control polygons through repeated local refinement \cite{Chaikin1974,Dyn1987,Cavaretta1991}. Among them, the four-point interpolatory scheme of Dyn, Gregory and Levin \cite{Dyn1987} is the most widely studied, and the six-point extension of Weissman \cite{Weissman1990} achieves $C^2$ limit curves. Both rules are governed by a single global scalar known as the tension parameter $\mu$, which is fixed at design time. Interpolatory schemes are particularly attractive in computer-aided design because they preserve every control vertex at every level of refinement, so that a designer who places a point can be confident that the limit curve passes through it.

A single global $\mu$ is, however, a real limitation when the control polygon mixes nearly straight segments with sharply turning corners. Each region calls for a different amount of tension, yet, the classical scheme applies one value throughout. Recent Euclidean interpolatory schemes that are explicitly adaptive, including the $\kappa$-curves of Yan et al.~\cite{Yan2017KCurves}, the $\varepsilon\kappa$-curves of Miura et al.~\cite{Miura2022EpsilonKappa}, Yuksel's $C^2$ interpolating splines~\cite{Yuksel2020}, and centripetal Catmull--Rom parameterisation \cite{CatmullRom1974,Yuksel2011}, demonstrate the empirical gains that local adaptivity can afford. Each of these schemes is, however, specific to the Euclidean plane. Curves on the two-sphere and in the Poincar\'e disk, which arise in satellite-track visualisation \cite{Cohen2018Spherical} and graph-layout systems \cite{Nickel2017,Chami2019}, are typically handled by manifold extensions of the classical fixed-tension rule \cite{Wallner2005,Xie2005,Sabin2005}.

In this paper, we ask whether a compact learned predictor can improve the classical fixed-tension angle-based subdivision rule while preserving structural interpolation, and whether such a predictor can be applied uniformly across the Euclidean plane, the two-sphere, and the Poincar\'e disk. Our entry point is the angle-based reformulation of interpolatory subdivision, in which the tension scalar is replaced by a per-edge insertion angle $\alpha_j$. The classical fixed-tension rule corresponds to one specific deterministic mapping from local exterior angles to $\alpha_j$. Allowing this mapping to be learned introduces local adaptivity while leaving the interpolation property structurally intact, because retained vertices are copied exactly regardless of any predicted angle.

The remainder of this paper develops that learned predictor, applies it across the three constant-curvature geometries by combining a shared predictor with geometry-specific geodesic primitives, and evaluates it against fixed-tension and curvature-adaptive baselines under a matched-density metric protocol. We do not claim this is a replacement for the broader class of classical interpolatory curve-design methods. The aim is to study one specific question within the angle-based subdivision framework and to report what the empirical evidence supports.

The contributions of this paper are fourfold. We introduce a learned local-angle selector, a compact edge-wise predictor that maps intrinsic local features to bounded insertion angles and is applied across constant-curvature manifolds by combining shared learned weights with geometry-specific geodesic primitives. We describe a structurally interpolatory implementation in which the subdivision rule copies the original vertices exactly, independent of the predicted angles, so that interpolation is a property of the operator rather than something learned. We present a corrected cross-method evaluation protocol that includes a validation-set fixed-$\mu$ oracle, a centripetal Catmull--Rom comparator on the Euclidean plane, a shared-versus-separate model comparison, structural invariant tests, and arc-length-matched smoothness metrics. Finally, we report an empirical trade-off analysis showing improved fidelity and smoothness behaviour over fixed global-tension subdivision, together with a transparent discussion of where the method is and is not preferable to existing techniques.

%%=============================================================================
\section{Related Work}
\label{sec:related}

The four-point scheme of Dyn, Gregory and Levin \cite{Dyn1987} and the six-point variant of Weissman \cite{Weissman1990} are the canonical interpolatory subdivision rules. Both apply a global tension parameter, with $\mu=0$ recovering the classical four-point rule and $\mu=-\tfrac{1}{4}$ yielding $C^2$ limit curves. The smoothness and convergence behaviour of these schemes has been analysed at length by Cavaretta, Dahmen and Micchelli \cite{Cavaretta1991}, by Hassan and Dodgson \cite{Hassan2003}, and in the textbook treatments of Farin \cite{Farin2002} and Prautzsch, Boehm and Paluszny \cite{Prautzsch2002}, but in every case the parameter remains spatially invariant. Convergence specifically on the sphere has been treated by H\"uning and Wallner \cite{Huning2022}.

A separate family of methods replaces fixed-tension subdivision with adaptive interpolating splines. The $\kappa$-curves of Yan et al.~\cite{Yan2017KCurves} construct $G^2$ interpolating splines whose curvature is monotone between consecutive control points and is bounded by the curvature at the endpoints. Yuksel's $C^2$ interpolating splines~\cite{Yuksel2020} extend this design with a higher continuity class. The $\varepsilon\kappa$-curves of Miura et al.~\cite{Miura2022EpsilonKappa} relax the monotonicity assumption to produce visually smoother interpolants with local control, and the smooth interpolating curves of Binninger and Sorkine-Hornung~\cite{Binninger2022Smooth} pursue a related design with monotone alternating curvature. Centripetal Catmull--Rom parameterisation \cite{CatmullRom1974,Yuksel2011} provides parameter-adaptive smoothness through arc-length-based knot placement. Related iterative approaches such as the accelerated progressive-iterative approximation of Yao and Hu \cite{Yao2025}, and constraint-based curve reconstruction \cite{Antony2025}, offer additional avenues for adaptive curve design with convergence guarantees. These schemes provide strong analytical guarantees on smoothness and curvature in Euclidean curve design. Our goal in this paper is narrower. We study learned local angle selection within an angle-based subdivision framework that also covers non-Euclidean constant-curvature geometries.

Subdivision on manifolds has its own substantial literature. Wallner and Dyn \cite{Wallner2005} introduced the log--exp generalisation of the four-point rule to Riemannian manifolds, with specialised variants developed for the sphere by Xie and Yu \cite{Xie2005}, Sabin and Dodgson \cite{Sabin2005} and H\"uning and Wallner \cite{Huning2022}, and for the hyperbolic plane by Ahanchaou and Ikemakhen \cite{Ahanchaou2022Hyperbolic}. Supporting all three constant-curvature regimes thus requires three distinct implementations, each tied to its own geodesic primitives. We retain those geometry-specific geodesic primitives but learn a single shared angle predictor that conditions on geometry through a small learned embedding. Related work has also explored sampling-uniformity guarantees for implicitly defined curves and surfaces \cite{Hu2025}, which informs the design of the dense reference sampling used during training. Refinement has also been studied in ambient spaces that are not Riemannian in the usual sense, including a recent subdivision scheme on the Heisenberg group trained with a central smoothness loss \cite{ugail2026subdiv}, which shares with the present work the strategy of leaving the refinement structure intact while allowing the rule that positions each new point to adapt to the geometry in which the curve lives.

A complementary tradition in geometric design obtains smooth shapes as solutions of elliptic partial differential equations rather than by repeated refinement of a control polygon. The PDE method of Ugail, Bloor and Wilson generates free-form surfaces from a small number of boundary curves together with a handful of design parameters \cite{ugail1999a}, and admits interactive reparameterisation of the resulting surface without disturbing its boundary conditions \cite{ugail1999b}. Higher-order formulations widen the reachable design space while retaining the same compact parameterisation \cite{kubiesa2004}, and the harmonic and biharmonic B\'ezier surfaces of Monterde and Ugail connect the PDE viewpoint directly to the polynomial patches of classical computer-aided design \cite{monterde2004,monterde2006}. Surveys of the area record its breadth across geometric design \cite{gonzalez2008}, and applications range from parametric aircraft geometry \cite{athanasopoulos2009} to patchwise approximation of large polygon meshes \cite{sheng2010} and facial geometry parameterisation \cite{sheng2011}. What this family shares with the present work is the aim of governing a smooth shape through a small number of meaningful quantities rather than through a dense set of unconstrained degrees of freedom. It differs in that the shape emerges from a boundary value problem, so passing through prescribed interior points is a constraint that must be imposed on the solution, whereas in the angle-based setting studied here interpolation is a structural property of the refinement operator itself.

Learning for geometric operators has emerged as a natural complement to classical subdivision. Liu et al.\ \cite{Liu2020NS} trained a network to predict new vertices for Euclidean Loop subdivision via per-patch weight sharing, with subsequent extensions to triangle meshes including neural progressive meshes \cite{Chen2023NPM} and implicit neural distance optimisation for mesh subdivision \cite{Liu2023INDI}. The broader literature on learning over discrete geometry covers point clouds and meshes \cite{Qi2017,Wang2019,Hanocka2019,Liu2024,Wang2025}, representation learning on non-Euclidean spaces \cite{Nickel2017,Chami2019,vanSpengler2023,He2025HyperbolicSurvey}, and neural operators between function spaces with applications in physical simulation and image-domain synthesis \cite{Li2021,Lu2021,Zhang2025,Hu2025b,Li2024}. Geometric deep learning \cite{Bronstein2021} provides the broader context. Our predictor is deliberately modest in scope. It is a small edge-wise multilayer perceptron acting on a fixed local stencil, not a general neural operator on curves, and we report its behaviour as a fidelity and smoothness trade-off rather than as a universal replacement for classical schemes.

%%=============================================================================
\section{Angle-Based Subdivision and Learned Local Angles}
\label{sec:method}

\subsection{Model spaces and the angle-based subdivision operator}

Let $\mM \in \{\Eucl, \Sph, \Hyp\}$ denote a two-dimensional constant-curvature model space equipped with its standard geodesic distance and exponential and logarithmic maps. We work with closed control polygons $\bP = (p_0, p_1, \dots, p_{N-1}) \in \mM^N$, with indices taken modulo $N$. A single interpolatory subdivision step produces a refined polygon $\bQ \in \mM^{2N}$ in which the original vertices are retained at even positions,
\begin{equation}
\label{eq:retain}
q_{2j} \;=\; p_j, \qquad j = 0, \dots, N-1,
\end{equation}
and the new vertices $q_{2j+1}$ are inserted between $p_j$ and $p_{j+1}$ along the geodesic through a local construction described below. After $k$ subdivision steps, the original vertices reappear at stride $2^k$ in the refined polygon. Equation~\eqref{eq:retain} is the structural interpolation invariant that the learned scheme preserves.

The local construction of each inserted vertex is governed by the intrinsic geometry of the control polygon. The geodesic edge length and the signed exterior angle at vertex $p_j$ are given by,
\begin{equation}
e_j \;=\; d_\mM(p_j, p_{j+1}), \qquad
\delta_j \;=\; \angle\bigl(\log_{p_j} p_{j-1},\, \log_{p_j} p_{j+1}\bigr) - \pi,
\end{equation}
where $d_\mM$ is the geodesic distance on $\mM$ and $\log$ is the manifold logarithm. The exterior angles $\delta_j$ are intrinsic, dimensionless and signed, and they take the same form across all three model spaces. The new vertex $q_{2j+1}$ is then parameterised by a single \emph{insertion angle} $\alpha_j$ that controls the deviation of $q_{2j+1}$ from the geodesic midpoint of $p_j$ and $p_{j+1}$ in the normal direction at that midpoint, namely,
\begin{equation}
\label{eq:insertion}
q_{2j+1} \;=\; \exp_{m_j}\bigl(\alpha_j\, n_j\bigr),
\qquad m_j = \text{geodesic midpoint of } p_j, p_{j+1},
\end{equation}
where $n_j$ is the unit normal to the geodesic at $m_j$. When $\alpha_j = 0$, the inserted vertex coincides with the geodesic midpoint, and the refined polygon traces the control polygon exactly. The classical four-point and six-point rules correspond to a specific deterministic assignment of $\alpha_j$ from the neighbouring exterior angles. With the tension parameter $\mu \in [-\tfrac{1}{2}, 0]$, the classical assignment may be written,
\begin{equation}
\label{eq:classical-mu}
\alpha_j^{\mathrm{cl}}(\mu) \;=\; \tfrac{1}{8}\bigl[\,\mu\,(\delta_{j-1} + \delta_{j+2}) \,+\, (1-\mu)(\delta_j + \delta_{j+1})\,\bigr],
\end{equation}
where the choices $\mu = 0$ and $\mu = -\tfrac{1}{4}$ recover the four-point and six-point schemes, respectively. Equation~\eqref{eq:classical-mu} is the baseline that the learned predictor replaces.

\subsection{The learned local-angle predictor}

We replace the deterministic mapping in Equation~\eqref{eq:classical-mu} with a compact learned function $f_\theta$ that produces a bounded per-edge angle from intrinsic local features. For the edge $(p_j, p_{j+1})$, the input feature vector is,
\begin{equation}
\bx_j \;=\; \biggl[\,\frac{\delta_{j-1}}{\pi},\, \frac{\delta_j}{\pi},\, \frac{\delta_{j+1}}{\pi},\, \frac{\delta_{j+2}}{\pi},\, \frac{e_j}{\bar e},\, \frac{e_{j+1}}{\bar e},\, \kappa\,\biggr],
\end{equation}
where $\bar e$ is the mean geodesic edge length over the polygon and $\kappa \in \{-1, 0, +1\}$ is a discrete code indicating the sign of the model-space curvature for $\Eucl$, $\Sph$ and $\Hyp$ respectively. A trainable embedding maps $\kappa$ to an eight-dimensional vector $\bz \in \R^8$, producing the conditioned input $[\bx_j, \bz] \in \R^{15}$. The predictor itself is a small multilayer perceptron with skip connections, LayerNorm \cite{Huang2023Normalization} and GELU activations \cite{Lee2023GELU}, with He initialisation \cite{He2015}, terminating in a bounded $\tanh$ activation,
\begin{equation}
\alpha_j \;=\; \alpha_{\max}\, \tanh\!\bigl( f_\theta([\bx_j, \bz]) \bigr), \qquad \alpha_{\max} = \pi / 4.
\end{equation}
The predictor is applied edge-wise and is therefore independent of the number of vertices in the control polygon. The output bound $\alpha_{\max} = \pi/4$ is a hard architectural safety constraint, although empirically, the network rarely operates near it. Detailed architecture and hyperparameter choices are reported in the supplementary material.

Interpolation remains structural in this construction. Only inserted vertices are predicted, and Equation~\eqref{eq:retain} holds exactly for any $\alpha_j$. The learned predictor cannot violate interpolation. It can only choose how each inserted vertex deviates from the geodesic midpoint.

\subsection{Training objective}

The predictor is trained on synthetic curves with known dense ground-truth sampling. Each training example consists of a control polygon $\bP$ and a ground-truth dense curve $\bG$ on the same manifold. The loss combines a symmetric Chamfer fidelity term \cite{Barrow1977}, a smoothness penalty on predicted exterior angles, and a bending penalty on the refined polygon,
\begin{equation}
\mL(\theta) \;=\; \mL_{\mathrm{chamfer}}(\bQ_\theta, \bG) \,+\, \lambda_s\, \mL_{\mathrm{smooth}}(\bQ_\theta) \,+\, \lambda_b\, \mL_{\mathrm{bend}}(\bQ_\theta).
\end{equation}
The Chamfer term uses the geometry-aware geodesic distance on $\mM$. The smoothness and bending terms are computed on the same arc-length-uniform resampling protocol used at evaluation time, which is described in Section~\ref{sec:protocol}. Because the input features are intrinsic and rotation-invariant by construction, no separate rotation-consistency regulariser is needed.

The dense reference curves used as training targets are intended as sparse-recovery references rather than as unique optimal fair interpolants. Their role is to define a controlled recovery task that is consistent across the three model spaces, allowing the same training objective to be applied uniformly. We return to the implications of this choice in Section~\ref{sec:discussion}.

%%=============================================================================
\section{Correctness and Evaluation Protocol}
\label{sec:protocol}

\subsection{Structural correctness of the learned scheme}

The subdivision operator constructed in Section~\ref{sec:method} preserves the interpolation property structurally, in the sense that retained vertices are copied exactly, regardless of the trained weights. We state this formally and verify it empirically before turning to the evaluation protocol.

\begin{proposition}
\label{prop:interp}
For any predicted insertion-angle sequence $\balpha = (\alpha_0, \dots, \alpha_{N-1})$, a single subdivision step satisfies $q_{2j} = p_j$ for all $j$. Consequently, after $k$ subdivision levels, the original control vertices appear at stride $2^k$ in the refined polygon, with zero positional error.
\end{proposition}

\begin{proof}
Let $S_{\balpha} : \mM^N \to \mM^{2N}$ denote the angle-based subdivision operator defined by Equations~\eqref{eq:retain} and~\eqref{eq:insertion}, parameterised by the insertion-angle sequence $\balpha$. The even-indexed entries of $\bQ = S_{\balpha}(\bP)$ are set by direct assignment, $q_{2j} = p_j$, with no reference to $\balpha$. The odd-indexed entries $q_{2j+1}$ are produced by Equation~\eqref{eq:insertion}, which writes only to output slots of odd indices. Because the index sets $\{2j\}$ and $\{2j+1\}$ are disjoint, the odd-index construction cannot perturb the even-index assignment, so $q_{2j} = p_j$ holds for every $j$ and every choice of $\balpha$. This proves the first claim.

Iterating the argument by induction on $k$, suppose the original vertex $p_j$ occupies index $j \cdot 2^k$ in the level-$k$ polygon $\bP^{(k)}$. Applying $S_{\balpha}$ sends each vertex at index $i$ of $\bP^{(k)}$ to index $2i$ of $\bP^{(k+1)}$ by the even-index rule, placing $p_j$ at index $j \cdot 2^{k+1}$. The base case $k=0$ is immediate, and so the claim holds for all $k \geq 0$. The retained vertices are placed by direct assignment rather than by any numerical operation, so the positional error in the geodesic distance on $\mM$ is exactly zero. The argument is independent of the values of $\balpha$ and therefore applies uniformly to the classical fixed-$\mu$ rule of Equation~\eqref{eq:classical-mu} and to the learned predictor of Section~\ref{sec:method}.
\end{proof}

The proposition is purely structural and holds for the learned predictor for the same reason it holds for the classical fixed-$\mu$ rule, namely that only odd-indexed vertices ever read from the predictor. A second structural safeguard is provided by the terminal $\tanh$ activation in the predictor, which confines the predicted insertion angle $\alpha_j$ to the interval $[-\alpha_{\max}, \alpha_{\max}]$ with $\alpha_{\max} = \pi/4$. This bound is a hard constraint enforced by the architecture and is independent of the weights $\theta$. It guarantees that each inserted vertex remains within a finite tubular neighbourhood of the geodesic midpoint, regardless of the outcome of training.

We verify Proposition~\ref{prop:interp} empirically on the validation set. Table~\ref{tab:invariants} reports four error quantities computed at full single-precision. The first column gives the one-step retained-vertex error after a single insertion step. The second and third columns give the same error after $k=5$ iterations of the classical fixed-$\mu$ rule and after $k=5$ iterations of the learned predictor. The fourth column gives the deviation between the inserted vertex and the true geodesic midpoint when $\alpha_j$ is forced to zero. Retained-vertex errors are bit-exact zero across all three geometries, and the zero-angle midpoint deviation is below $3\times 10^{-7}$ in each geometry, on the order of single-precision arithmetic noise.

\begin{table}[t]
\centering
\caption{Implementation invariants verified on the validation set. Retained-vertex errors are bit-exact zero. The zero-angle midpoint deviation is at single-precision noise level.}
\label{tab:invariants}
\small
\begin{tabular}{lcccc}
\toprule
Geometry & One-step & Repeated classical ($k=5$) & Repeated neural ($k=5$) & Zero-angle midpoint \\
\midrule
$\Eucl$ & $0.0$ & $0.0$ & $0.0$ & $< 3\times 10^{-7}$ \\
$\Sph$  & $0.0$ & $0.0$ & $0.0$ & $< 3\times 10^{-7}$ \\
$\Hyp$  & $0.0$ & $0.0$ & $0.0$ & $< 3\times 10^{-7}$ \\
\bottomrule
\end{tabular}
\end{table}

\subsection{Matched-density evaluation protocol and reference curves}

Discrete tangent and bending measures depend sensitively on sampling density and on sample placement along the curve. To make the comparison across methods fair, every output curve is resampled before any smoothness metric is computed. We resample to $N \cdot 2^k$ points using geometry-aware arc-length interpolation, where $N$ is the number of control vertices and $k$ is the number of subdivision levels applied. The same resampled curve is also used for the reported fidelity metrics, namely the mean nearest-neighbour distance and the symmetric Hausdorff distance. This matched-density protocol is essential, because, without it, the apparent bending energy of centripetal Catmull--Rom on $\Eucl$ is inflated by approximately two orders of magnitude due solely to vertex clustering at segment boundaries, even though the underlying curve is smooth. Parametric baselines and recursive subdivision schemes produce different native sampling densities, and the resampling step is what allows them to be compared on a common footing.

The synthetic reference curves used for evaluation are not treated as unique optimal fair interpolants through the control points. They provide controlled recovery tasks from sparse samples, allowing consistent measurement of approximation error and smoothness across methods. We make this point explicit in the discussion in Section~\ref{sec:discussion}, because the choice of reference family influences what the predictor learns and what the comparison can meaningfully claim.

%%=============================================================================
\section{Experiments}
\label{sec:experiments}

\subsection{Setup and baselines}

We use $N = 12$ control vertices and $k = 5$ subdivision levels throughout, producing refined polygons with $N \cdot 2^k = 384$ vertices. Training data is generated synthetically and separately for each geometry. The Euclidean training family combines ellipses and Fourier-perturbed curves. The spherical family combines perturbed great circles, polar Fourier curves, and Lissajous tracks. The hyperbolic family uses analogous parametric constructions in the Poincar\'e disk. Each geometry contributes 120 curves, split into training and validation subsets with a single fixed seed offset and disjoint curve families per split. The same validation construction is used for every number reported in this section. Training uses AdamW \cite{Loshchilov2019} with $\beta_1=0.9$, $\beta_2=0.95$ and weight decay $10^{-4}$, run with three independent random initialisations (seeds $7$, $11$, $19$) so that the cross-seed mean and standard deviation can be reported. Riemannian variants of adaptive optimisation \cite{Becigneul2019} were considered but were not required, because the network parameters live in standard Euclidean weight space. The shared geometry-conditioned model has $28{,}737$ parameters, and training takes approximately six minutes per seed on a single consumer GPU.

The learned predictor is evaluated against four baselines drawn from three distinct families. The first family is the canonical fixed-tension interpolatory schemes, namely four-point ($\mu = 0$) and six-point ($\mu = -\tfrac{1}{4}$) following Dyn, Gregory and Levin \cite{Dyn1987} and Weissman \cite{Weissman1990}. These rules are applied to all three geometries via log--exp insertion as introduced by Wallner and Dyn \cite{Wallner2005}. The second baseline is a validation-set fixed-$\mu$ oracle, in which a single global tension value per geometry is chosen by minimising mean-NN error on the validation set itself, with $\mu$ swept over a 27-point grid covering the interval from $-\tfrac{1}{2}$ to $0.15$. This oracle represents the strongest single-parameter fixed-tension baseline that can be constructed and is a more demanding comparator than either canonical choice. The third baseline is centripetal Catmull--Rom parameterisation \cite{Yuksel2011}, included on the Euclidean plane only as a modern parameter-adaptive interpolating spline. We do not extend it to the sphere or to the hyperbolic plane in this paper. The fourth and final method is the learned predictor itself, namely the shared geometry-conditioned network described in Section~\ref{sec:method}. The shared-versus-separate comparison in Section~\ref{sec:shared-vs-separate} adds independently trained per-geometry predictors.

\begin{figure}[t]
\centering
\includegraphics[width=0.98\textwidth]{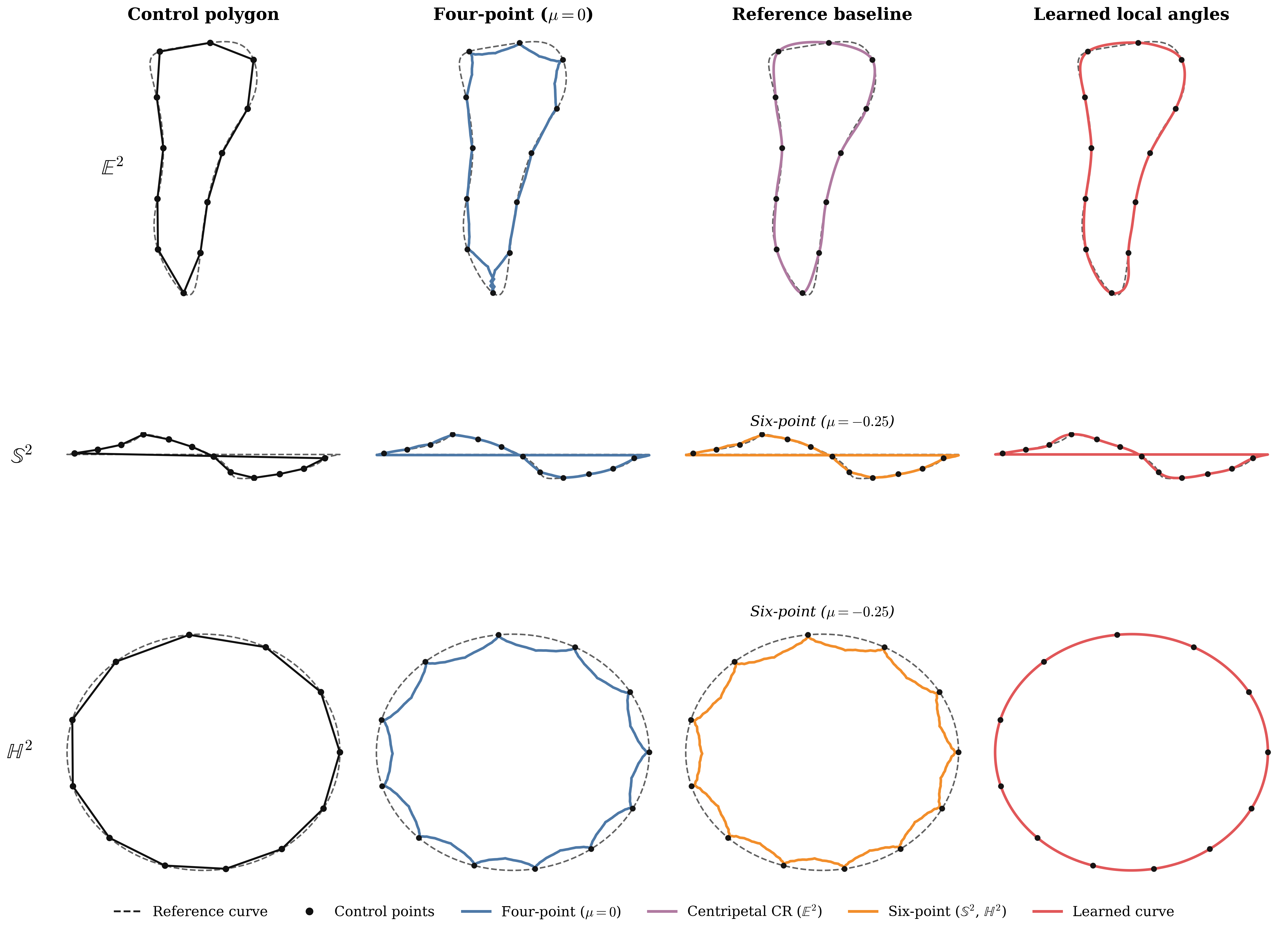}
\caption{Qualitative comparison on representative validation curves. The three rows correspond to $\Eucl$, $\Sph$ and $\Hyp$. The columns show, from left to right, the control polygon, the four-point output at $\mu=0$, a strong-baseline reference (centripetal Catmull--Rom on $\Eucl$ and six-point at $\mu=-0.25$ on $\Sph$ and $\Hyp$), and the learned local-angle output. Black dots are the control points, which lie on all rendered curves. The dashed line is the dense reference. Fixed-tension schemes scallop at the control vertices, while the learned scheme produces smooth interpolating curves.}
\label{fig:qual}
\end{figure}

\subsection{Qualitative comparison}

Figure~\ref{fig:qual} shows representative validation curves on each geometry, namely a high-aspect-ratio Euclidean shape, a perturbed great circle on the sphere, and a hyperbolic ring near the Poincar\'e disk boundary. Across the three geometries, the visible failure mode of the fixed-tension schemes is a regular scalloping at the control points, in which the refined polygon oscillates noticeably between consecutive vertices. The learned predictor passes through the same control points smoothly, with no comparable scalloping. The control points themselves lie on every curve in the figure, providing direct visual confirmation of the structural interpolation property established in Section~\ref{sec:protocol}.

\begin{figure}[t]
\centering
\includegraphics[width=0.98\textwidth]{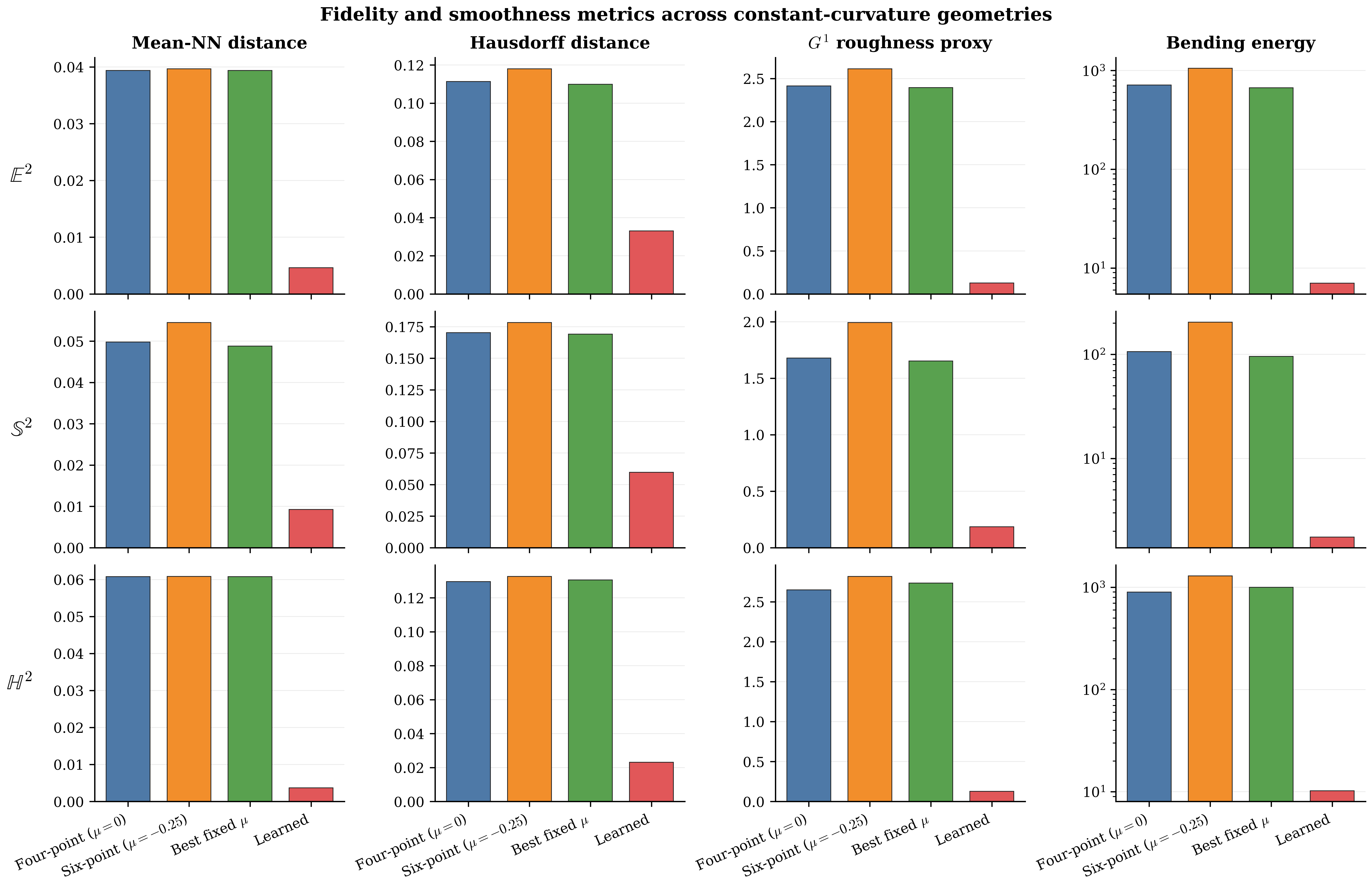}
\caption{Fidelity and smoothness metrics across constant-curvature geometries. Rows correspond to $\Eucl$, $\Sph$ and $\Hyp$. Columns correspond to mean nearest-neighbour distance, Hausdorff distance, the $G^1$ roughness proxy, and bending energy. Bars show the cross-seed mean per method, and bending is shown on a logarithmic scale. The learned predictor improves all four metrics over the validation-tuned fixed-$\mu$ baseline on every geometry.}
\label{fig:metrics}
\end{figure}

\begin{table}[t]
\centering
\caption{Main quantitative results across the three constant-curvature geometries, reported as a cross-seed summary. Each cell reports the per-curve metric averaged over all validation curves of that geometry, then summarised as mean and standard deviation across three independent training seeds. The released reproduction script also writes the corresponding per-seed and aggregate results files. The best values per metric per geometry are in bold. Note, the bending energy on $\Sph$ and $\Hyp$ is a within-geometry diagnostic, and absolute values are not directly comparable across geometries.}
\label{tab:main}
\footnotesize
\setlength{\tabcolsep}{4pt}
\begin{tabular}{llcccc}
\toprule
& Method & Mean-NN $\downarrow$ & Hausdorff $\downarrow$ & $G^1$ proxy $\downarrow$ & Bending $\downarrow$ \\
\midrule
\multirow{5}{*}{$\Eucl$}
 & Four-point ($\mu=0$)            & $0.0366 \pm 0.0019$ & $0.124 \pm 0.012$ & $2.43 \pm 0.17$ & $973 \pm 207$ \\
 & Six-point ($\mu=-0.25$)         & $0.0372 \pm 0.0017$ & $0.129 \pm 0.010$ & $2.57 \pm 0.16$ & $1508 \pm 301$ \\
 & Best fixed $\mu$                & $0.0365 \pm 0.0020$ & $0.121 \pm 0.013$ & $2.34 \pm 0.16$ & $811 \pm 150$ \\
 & Centripetal Catmull--Rom        & $0.0109 \pm 0.0007$ & $0.065 \pm 0.020$ & $0.43 \pm 0.13$ & $27.3 \pm 9.2$ \\
 & \textbf{Learned}                & $\mathbf{0.0060 \pm 0.0013}$ & $\mathbf{0.051 \pm 0.016}$ & $\mathbf{0.20 \pm 0.07}$ & $\mathbf{12.4 \pm 4.8}$ \\
\midrule
\multirow{4}{*}{$\Sph$}
 & Four-point                      & $0.0567 \pm 0.0142$ & $0.200 \pm 0.050$ & $1.78 \pm 0.25$ & $121 \pm 22$ \\
 & Six-point                       & $0.0617 \pm 0.0153$ & $0.207 \pm 0.050$ & $2.00 \pm 0.21$ & $215 \pm 16$ \\
 & Best fixed $\mu$                & $0.0537 \pm 0.0138$ & $0.194 \pm 0.047$ & $1.71 \pm 0.31$ & $91 \pm 16$ \\
 & \textbf{Learned}                & $\mathbf{0.0115 \pm 0.0017}$ & $\mathbf{0.065 \pm 0.005}$ & $\mathbf{0.18 \pm 0.02}$ & $\mathbf{1.85 \pm 0.27}$ \\
\midrule
\multirow{4}{*}{$\Hyp$}
 & Four-point                      & $0.0583 \pm 0.0032$ & $0.127 \pm 0.004$ & $2.63 \pm 0.04$ & $969 \pm 99$ \\
 & Six-point                       & $0.0581 \pm 0.0032$ & $0.129 \pm 0.004$ & $2.82 \pm 0.06$ & $1509 \pm 131$ \\
 & Best fixed $\mu$                & $0.0579 \pm 0.0033$ & $0.130 \pm 0.003$ & $2.84 \pm 0.13$ & $1681 \pm 492$ \\
 & \textbf{Learned}                & $\mathbf{0.0034 \pm 0.0009}$ & $\mathbf{0.027 \pm 0.009}$ & $\mathbf{0.11 \pm 0.03}$ & $\mathbf{8.5 \pm 2.8}$ \\
\bottomrule
\end{tabular}
\end{table}

\subsection{Main quantitative results}

Table~\ref{tab:main} reports mean nearest-neighbour distance, Hausdorff distance, the $G^1$ roughness proxy and bending energy, each averaged over the validation set and reported as cross-seed mean and standard deviation. The best values per metric per geometry are shown in bold. Figure~\ref{fig:metrics} visualises the same data on common axes per metric. The learned predictor improves mean-NN error by a factor of $6.1$ on the Euclidean plane, $4.7$ on the sphere, and $17.0$ on the hyperbolic plane, in each case relative to the best validation-tuned fixed-$\mu$ baseline. On the Euclidean plane, it also improves mean-NN error by a factor of $1.8$ relative to centripetal Catmull--Rom. Across all three geometries, bending energy and $G^1$ roughness drop by one to three orders of magnitude relative to fixed-tension baselines.

\begin{figure}[t]
\centering
\includegraphics[width=0.95\columnwidth]{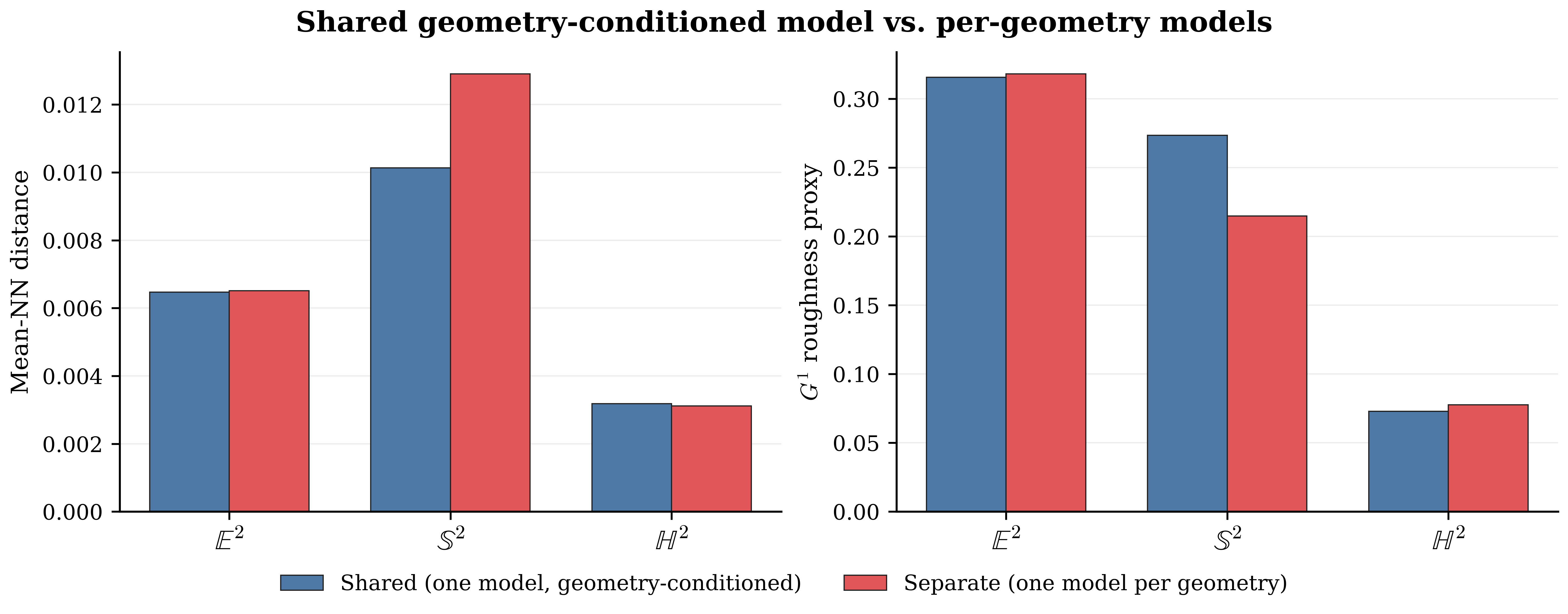}
\caption{Shared geometry-conditioned model versus three separate per-geometry models on mean nearest-neighbour distance and on the $G^1$ roughness proxy. The shared model is competitive on mean-NN across all geometries, while on the $G^1$ proxy the separate models do slightly better on $\Sph$.}
\label{fig:shared}
\end{figure}

\begin{table}[t]
\centering
\caption{Shared versus separate networks, single seed. Shared denotes one geometry-conditioned predictor with $28{,}737$ parameters. Separate denotes three independently trained predictors of the same architecture, one per geometry, each with $28{,}737$ parameters, so that the separate condition requires three deployed sets of weights with $86{,}211$ parameters in total.}
\label{tab:shared}
\small
\begin{tabular}{llcccc}
\toprule
Condition & Geometry & Mean-NN & Hausdorff & $G^1$ & Bending \\
\midrule
\multirow{3}{*}{Shared}   & $\Eucl$ & $0.0065$ & $0.0701$ & $0.325$ & $20.63$ \\
                          & $\Sph$  & $0.0099$ & $0.0709$ & $0.274$ & $\phantom{0}1.85$ \\
                          & $\Hyp$  & $0.0031$ & $0.0218$ & $0.072$ & $\phantom{0}3.85$ \\
\midrule
\multirow{3}{*}{Separate} & $\Eucl$ & $0.0065$ & $0.0703$ & $0.290$ & $19.18$ \\
                          & $\Sph$  & $0.0128$ & $0.0887$ & $0.221$ & $\phantom{0}1.59$ \\
                          & $\Hyp$  & $0.0032$ & $0.0204$ & $0.086$ & $\phantom{0}4.07$ \\
\bottomrule
\end{tabular}
\end{table}

\subsection{Shared versus separate models}
\label{sec:shared-vs-separate}

A natural question is whether a single geometry-conditioned network is genuinely useful, or whether three independently trained per-geometry predictors would do better. Figure~\ref{fig:shared} and Table~\ref{tab:shared} compare a single shared geometry-conditioned predictor against three independently trained per-geometry predictors of the same architecture. Each individual predictor has $28{,}737$ parameters, so the separate condition requires three deployed sets of weights with $86{,}211$ parameters in total, while the shared condition uses a single set of $28{,}737$ parameters covering all three geometries. The shared model is competitive on mean-NN across all geometries and is noticeably better on the sphere. The separate models obtain slightly lower $G^1$ roughness on the Euclidean plane and on the sphere, while the two conditions are essentially equal on the hyperbolic plane. We therefore do not claim uniform superiority for shared training. The result supports the shared predictor as a compact cross-geometry implementation that recovers most of the per-geometry performance at one third of the deployed parameter budget.

\subsection{Out-of-distribution case study on the ISS ground track}

To probe behaviour outside the synthetic training distribution, we evaluate on a real spherical curve, namely the ground track of the International Space Station over one orbital period. The track was sampled at $N=16$ control points and refined to $N\cdot 2^5 = 512$ output points using $k=5$ subdivision levels. The training set contained no curves drawn from this family. 

\begin{figure}[t]
\centering
\includegraphics[width=0.98\textwidth]{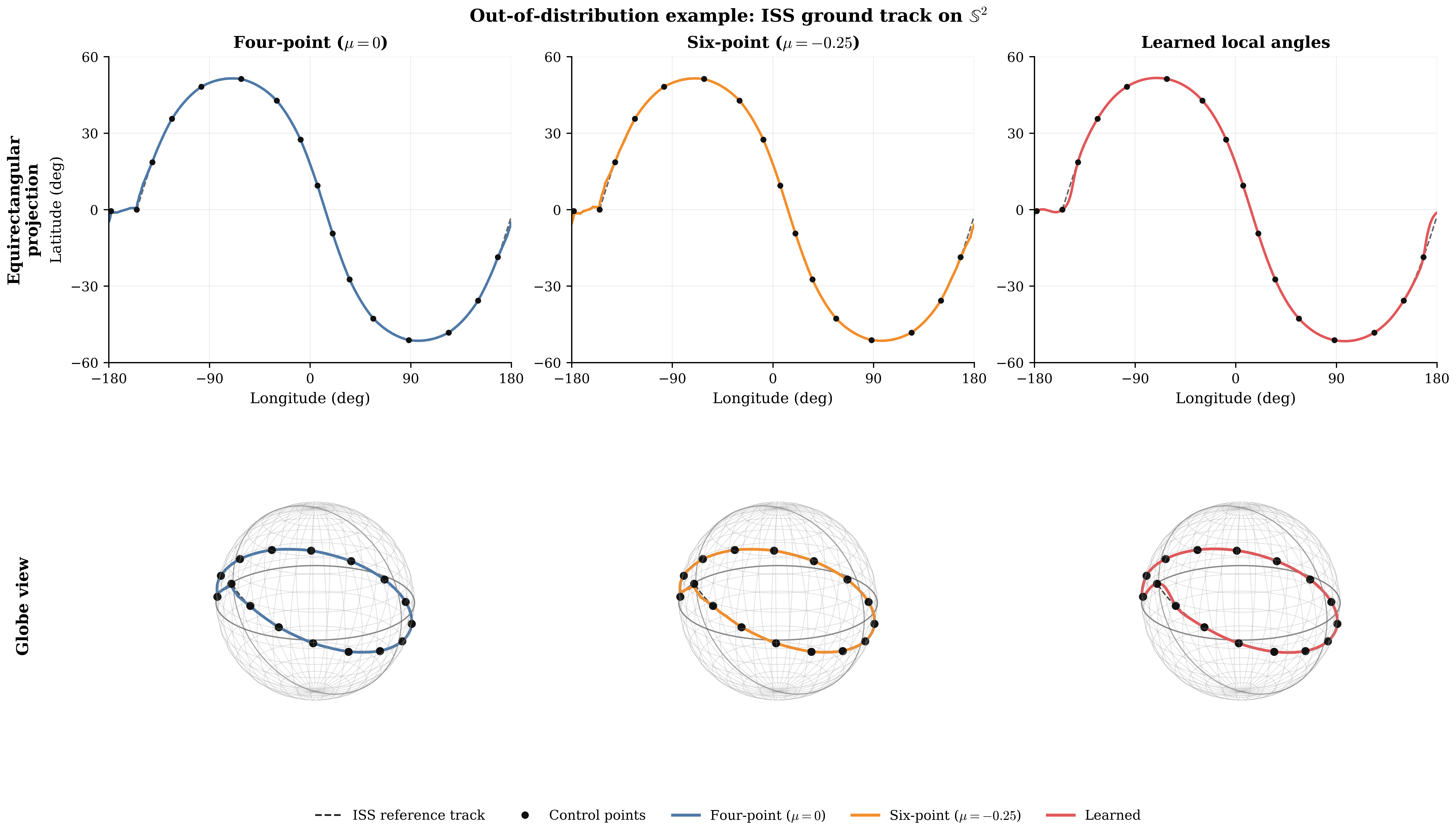}
\caption{Out-of-distribution example on the ISS orbital ground track. The top row shows an equirectangular projection. The bottom row shows the same outputs on a globe view. The learned predictor follows the reference track with comparable fidelity to the fixed-tension baselines while producing a substantially smoother refined curve.}
\label{fig:iss}
\end{figure}

\begin{table}[t]
\centering
\caption{Out-of-distribution metrics on the ISS ground track. Results are for a single curve and are not averaged across seeds.}
\label{tab:iss}
\small
\begin{tabular}{lcccc}
\toprule
Method & Mean-NN & Hausdorff & $G^1$ & Bending \\
\midrule
Four-point ($\mu = 0$)        & $0.0118$ & $0.202$ & $2.24$ & $114.4$ \\
Six-point ($\mu = -0.25$)     & $0.0129$ & $0.204$ & $2.91$ & $317.4$ \\
\textbf{Learned}              & $0.0135$ & $0.200$ & $\mathbf{0.08}$ & $\mathbf{1.73}$ \\
\bottomrule
\end{tabular}
\end{table}

Figure~\ref{fig:iss} shows the equirectangular and globe-view comparisons, and Table~\ref{tab:iss} reports the corresponding metrics. Fidelity, measured by mean-NN and Hausdorff distances, is comparable across all three methods on this real out-of-distribution trajectory. The learned predictor reduces bending energy by approximately a factor of $66$ relative to four-point and $184$ relative to six-point, with roughly $28$-fold and $36$-fold reductions in the $G^1$ roughness proxy, respectively. We interpret this result as a smoothness diagnostic on out-of-distribution input rather than as evidence of universal superiority.

A final remark here concerns the interpretation of bending energy across geometries. Bending energy is computed as a discrete proxy for $\int \kappa^2\, ds$, and its absolute value depends on the model space, the local arc-length scale, and the chosen sampling. We therefore report it as a within-geometry diagnostic for comparing methods rather than as a quantity to be compared across $\Eucl$, $\Sph$ and $\Hyp$. The matched-density resampling protocol introduced in Section~\ref{sec:protocol} is what makes the within-geometry comparison meaningful. Cross-geometry comparison of bending values is not appropriate and is not implied by any claim made in this paper.

%%=============================================================================
\section{Discussion and Limitations}
\label{sec:discussion}

The learned predictor improves the fidelity and smoothness trade-off of interpolatory subdivision relative to fixed global-tension baselines across the three constant-curvature geometries considered, while preserving structural interpolation by construction. On the Euclidean plane, where centripetal Catmull--Rom parameterisation provides a strong adaptive comparator, the learned predictor improves on Catmull--Rom in mean-NN error by a factor of $1.8$ and in bending energy by roughly a factor of $2.2$, with comparable $G^1$ roughness. The shared geometry-conditioned network remains competitive with per-geometry specialists while using one deployed set of weights rather than three, making it a practical choice for cross-geometry deployment.

We deliberately stop short of claiming that the learned predictor introduced here is a replacement for the broader class of classical interpolatory curve-design methods. The $\kappa$-curves of Yan et al.~\cite{Yan2017KCurves}, Yuksel's $C^2$ interpolating splines~\cite{Yuksel2020}, the $\varepsilon\kappa$-curves of Miura et al.~\cite{Miura2022EpsilonKappa}, the smooth interpolating curves of Binninger and Sorkine-Hornung~\cite{Binninger2022Smooth}, and the wider literature on smooth interpolating splines offer strong analytical guarantees on curvature behaviour that the learned predictor does not match. The scheme presented here is best understood as an adaptive subdivision operator within the angle-based framework. Other framings, including those with formal curvature guarantees, remain valuable and complementary, and the choice between them in any given application will depend on whether analytical guarantees or learned adaptivity is the more important consideration. We also do not claim that synthetic dense reference curves define the unique fair interpolant through a given control polygon. Those references define a consistent sparse-recovery task across geometries, which permits uniform measurement of approximation error and smoothness, but the choice of reference family influences what the predictor learns, and we avoid framing the comparison as the recovery of a uniquely correct curve.

Several limitations of the present work deserve explicit acknowledgement. The method is currently demonstrated on closed curves with a fixed number of control vertices at evaluation time, although the edge-wise predictor itself is independent of polygon length and applies to arbitrary closed control polygons of the kind used in computer-aided design. Open curves require boundary stencil treatment that we have not implemented in this paper, although such treatment is a standard extension of interpolatory subdivision. The training procedure also requires synthetic dense reference curves, and the method does not yet offer a pathway to unsupervised refinement from control polygons alone. The structural interpolation guarantee depends on the architecture rather than on the trained weights, and a different reformulation that wrote even-indexed vertices through the network would lose this property, so the design choice is load-bearing. We have verified retained-vertex interpolation to machine precision, but we do not provide formal convergence guarantees for the trained operator, since convergence of the limit curve depends on the regularity of the learned mapping. Our input features are rotation-invariant by construction, although a tighter equivariance accounting for the full trained operator using formal rotation metrics \cite{Huynh2009} remains for future work. Finally, bending energy is a discrete diagnostic and is only meaningful within a fixed sampling density. The matched-density protocol of Section~\ref{sec:protocol} enables within-geometry comparison but does not justify cross-geometry comparison of bending values.

Natural extensions of this work include open curves with explicit boundary stencils, surfaces (subdivision schemes on the faces of triangle meshes), and unsupervised or semi-supervised training procedures that do not require dense reference curves. A formal connection between the regularity of the learned predictor and the smoothness of the limit curve would strengthen the theoretical picture.

%%=============================================================================
\section{Conclusion}
\label{sec:conclusion}

In this paper, we have introduced a learned local-angle formulation for interpolatory curve subdivision in three constant-curvature geometries, namely, the Euclidean plane, the two-sphere, and the Poincar\'e disk. Rather than replacing the interpolatory subdivision rule itself, the proposed method learns the per-edge insertion angle that controls each newly inserted vertex. The original vertices are still copied exactly at every refinement level, so interpolation remains a structural property of the operator and is independent of the trained weights.

The experiments show that local angle prediction provides a consistent improvement over fixed global-tension subdivision. Under a matched-density evaluation protocol, the learned predictor improves the fidelity--smoothness trade-off relative to the validation-tuned fixed-\(\mu\) baseline across all three geometries. The Euclidean comparison with centripetal Catmull--Rom further indicates that the method remains competitive against a strong adaptive interpolatory curve model. The shared geometry-conditioned predictor is also competitive with separately trained per-geometry models, supporting the use of a single compact model together with geometry-specific geodesic primitives.

The method should be understood as an adaptive operator within the angle-based subdivision framework, not as a replacement for all classical interpolatory curve-design techniques. In particular, it does not provide the formal curvature guarantees available for some spline-based methods, and the present study is limited to closed curves with supervised training from dense reference samples. Future work should address open-curve boundary stencils, surface subdivision, weaker forms of supervision, and theoretical links between the regularity of the learned angle predictor and the smoothness of the resulting limit curve.

%%=============================================================================
\section*{Acknowledgments}
The authors acknowledge the computational resources provided by the Centre for
Visual Computing and Intelligent Systems at the University of Bradford.

\section*{Data availability}
The full codebase and data used for this study are available at: \\
\url{https://github.com/ugail/adaptive-interpolatory-subdivision} \\
(Zenodo doi: \url{https://doi.org/10.5281/zenodo.20426602}) \\
The repository contains the Python notebooks and scripts used for training,
evaluation, shared-versus-separate model comparison, and invariant testing. It
also contains the synthetic dataset generators, trained model checkpoints,
configuration files, and generated result files used in this paper. The
repository is designed to reproduce all reported numerical results, figures,
and tables from a clean run.

\section*{Funding}
No funding was received for this work.

\section*{Competing interests}
The authors declare no competing interests.

\section*{Author contributions}
Both authors contributed equally to this work.

%=======================================================================

\end{document}